\documentclass[11pt]{article}
\usepackage{fullpage}
\usepackage{wustyle}

\usepackage{braket}
\usepackage[top=0.9in,bottom=1.1in,left=1.1in,right=1.1in]{geometry}
\usepackage{authblk}

\title{Convergence Analysis of Discrete Diffusion Model: \\ Exact Implementation through Uniformization}

\author{Hongrui Chen\thanks{Stanford University, \href{hongrui@stanford.edu}{hongrui@stanford.edu} }
\quad
Lexing Ying \thanks{Stanford University, \href{lexing@stanford.edu}{lexing@stanford.edu}}
}

\date{\today}

\begin{document}

\maketitle
\begin{abstract}
Diffusion models have achieved huge empirical success in data generation tasks. Recently, some efforts have been made to adapt the framework of diffusion models to discrete state space, providing a more natural approach for modeling intrinsically discrete data, such as language and graphs. This is achieved by formulating both the forward noising process and the corresponding reversed process as Continuous Time Markov Chains (CTMCs). In this paper, we investigate the theoretical properties of the discrete diffusion model. Specifically, we introduce an algorithm leveraging the uniformization of continuous Markov chains, implementing transitions on random time points. Under reasonable assumptions on the learning of the discrete score function, we derive Total Variation distance and KL divergence guarantees for sampling from any distribution on a hypercube. Our results align with state-of-the-art achievements for diffusion models in $\RR^d$ and further underscore the advantages of discrete diffusion models in comparison to the $\RR^d$ setting.
\end{abstract}
\section{Introduction}
Generative modeling is one of the central tasks in machine learning, which aims to learn a probability distribution from data and generate data from the learned distribution. The diffusion model has emerged as a powerful and versatile framework in generative modeling, achieving state-of-the-art performance in a variety of data generation tasks, including image generation \cite{nichol2021improved,bartal2023multidiffusion}, audio generation \cite{schneider2023archisound}, video generation \cite{yang2022diffusion,ho2022video}, text-to-image synthesis \cite{ramesh2021zeroshot, ramesh2022hierarchical}, and computational biology \cite{guo2023diffusion}. The general framework of the score-based generative model involves 1) defining a forward noising process to gradually diffuse the data distribution to some simple distribution (like standard Gaussian); 2) learning a reversed process to denoising the simple distribution to the data distribution by estimating the score functions of the forward diffusion process. 

Works on the diffusion model focus on the forward processes defined in the Euclidean state space $\RR^d$. In such scenarios, an ideal choice of the forward process is the Ornstein-Uhlenbeck (OU) process, which is driven by a stochastic differential equation(SDE) on $\RR^d$, and the corresponding reversed process is also given by an SDE. Nevertheless, certain data generation tasks present an intrinsic characteristic of discrete data. For example, natural language processing operates within a discrete token space; computer vision involves discrete representations of images; and molecular graph modeling engages with graph data in a discrete structure \cite{hoogeboom2021argmax, zheng2023reparameterized,huang2023conditional}.
Thus, it is more natural to use diffusion processes on the discrete state space to model these discrete data distributions. 

To this end, some recent works \cite{meng2023concrete, NEURIPS2022_b5b52876, benton2023denoising, santos2023blackout, lou2023discrete} have introduced a framework for diffusion models in discrete state spaces. This framework notably utilizes a continuous-time Markov chain (CTMC) in the discrete state space for the forward process, and the corresponding reverse process is also a CTMC. Moreover, mirroring the concept of score estimation in diffusion models on $\RR^d$, they proposed a discrete score function given by the ratios of probability mass on different states, and the score entropy loss as a new score matching objective that is derived from KL divergence between the path measures of the forward and the reversed process. Combining the learning of the discrete score function through minimizing the score entropy and the sampling from the learned reversed process, a completed procedure for the diffusion model on discrete state space has been established.

However, despite the potential advantage of the discrete diffusion model, unlike the extensively studied SDE framework, the theoretical understanding of the CTMC framework has not been built. A line of works \cite{convergence-score, convergencescore2, SamplingEasy, chen2023improved, li2023faster, benton2023linear} concentrated on the theory of diffusion model on $\RR^d$. Generally speaking, the established theoretical results can be summarized as follows:
\begin{itemize}
\item Sampling is as easy as learning the score: for arbitrary data distribution, suppose one can estimate the score function at multiple noise levels, then one can approximately sample from the data distribution.
\item Quantitatively, under an $L^2$ accurate score estimator on the forward process, $O\left(\frac{d\log(1/\delta)}{\epsilon^2}\right)$ iterations suffices to output a distribution that is $\epsilon^2$-close in KL divergence to a distribution $p_\delta$, where $p_\delta$ is a variance-$\delta$ Gaussian perturbation of the data distribution.
\item There are three sources of error in the diffusion model: 1)the error from the inexact score estimator, 2)the error from insufficient mixing of the forward process, and 3)the discretization error. The discretization error causes the key challenges in the analysis due to the error propagation in the numerical simulation of a non-contractive dynamic.
\end{itemize} 
In this paper, we take a step toward the theory of diffusion model in the CTMC framework and aim to understand how the theoretical property of discrete diffusion compares to the established theory for diffusion model on $\RR^d$. Our results suggest that:
\begin{itemize}
\item One can implement the reversed CTMC in an \emph{exact} way, i.e., without discretization error, through an algorithm based on the uniformization technique \cite{GRASSMANN197747,VANDIJK1992339,deSouzaeSilva2000}. This presents a surprising advantage of the CTMC framework compared to the SDE framework, where discrete errors are significant in the analysis.
\item The proposed algorithm is guaranteed to be efficient by our theoretical analysis. Quantitatively, assuming an $\epsilon$-error bound on the score entropy loss,  our purposed algorithm requires $O(d\log(d/\epsilon \delta))$ steps to reach a distribution $\epsilon$-close in KL divergence to a distribution $p_\delta$, where $p_\delta$ is a $\delta$-perturbation of the data distribution in TV distance. Combining the two errors, we also get a TV distance guarantee for sampling from the data distribution. Thus, we obtain a logarithm dependence on $\epsilon$ for the discrete diffusion model, in contrast to the $\epsilon^{-2}$ dependence for the SDE framework. 
\end{itemize}

\paragraph{Organization of this paper}
In Section \ref{sc2}, we introduce some preliminary background on CTMC and the uniformization technique for CTMC. In Section\ref{sc3}, we introduce the framework of the diffusion model using CTMC on discrete state space; in Section \ref{sc4}, we propose our new algorithm for numerical simulation of the reversed CTMC process and analyze the complexity and error of the proposed algorithm.

\subsection{Related works}
\paragraph*{Diffusion Models on Discrete State Space}
In the initial stages of development, the diffusion model was formulated as a discrete-time Markov process. It was first introduced in \cite{pmlr-v37-sohl-dickstein15}, marking the pioneering work in applying the diffusion process for data generation tasks. Although the focus has primarily been on continuous-space formulations, the discrete-time and discrete-space framework of the diffusion model was also described in the initial work \cite{pmlr-v37-sohl-dickstein15} and then further explored in \cite{seff2019discrete,hoogeboom2021argmax,austin2023structured}.
 Some works, such as \cite{niu2020permutation,song2021maximum}, tackle the discrete data by embedding it to $\RR^d$ and employing the continuous-space diffusion model framework with dequantization techniques. This approach has also been popular in other generative models like Variational Autoencoders (VAE) and auto-regressive models.

A significant breakthrough in diffusion models was achieved when Song et al. in \cite{song2021maximum, SGMSDEsong} introduced a continuous framework through Stochastic Differential Equations. 
As a discrete analogy of the SDE framework,  firstly established the CTMC framework for diffusion models on discrete state space. This paper introduces an ELBO objective derived from the KL divergence between path measures and proposes multiple numerical methods for simulating the reversed process, including $\tau$-leaping and prediction-correction.
 \cite{meng2023concrete} proposed the concrete score matching method to learn the probability ratio term in the reversed CTMC. However, this objective does not align with the path-KL and might lack stability in practice.
 \cite{sun2023scorebased} alternatively derive the score-matching objective from the conditional marginal distributions. 
\cite{santos2023blackout} proposed a different, forward process called blackout diffusion, which transforms data to an empty set instead of uniform distribution. \cite{zheng2023reparameterized,lou2023discrete} applies the continuous-time discrete diffusion model to generate language data and, in particular, \cite{lou2023discrete} scale the method to GPT2 data, demonstrating its potential effectiveness in handling large datasets.

\paragraph*{Theory of Diffusion Models}
This paper is closely related to a series of works \cite{convergence-score, Bortoli2022ConvergenceOD, convergencescore2, SamplingEasy, chen2023improved, li2023faster, benton2023linear} focused on the theoretical analysis of diffusion models in $\RR^d$. Specifically, these studies seek to answer the following question: given an $L^2$-accurate score estimator, how closely does the distribution generated by the reverse SDE with the score estimator, in place of the actual score function, and with appropriate discretization, approximate the data distribution? This question was first addressed for smooth and isoperimetric distributions in \cite{convergence-score}, followed by a reduction of the isoperimetry assumption in \cite{SamplingEasy} and the smoothness assumption in \cite{convergencescore2, chen2023improved}. The state-of-art result, which is applicable to any distribution and shows a nearly linear dependence on the dimension $d$, is provided in \cite{benton2023linear}. In this paper, we answer this question for the discrete diffusion model. Our results match the state-of-the-art theory in \cite{benton2023linear} for the $\RR^d$ setting, applying to any distribution on the hypercube and exhibiting a nearly linear dependence on $d$.

\cite{NEURIPS2022_b5b52876} also provides an error bound for the discrete diffusion model. However, this analysis relies on some strong assumptions like the $L^\infty$ accuracy of the score estimator and the bounded probability ratio of the data distribution. In addition, their result also has a sub-optimal quadratic dependence on the dimension $d$. In this paper, we will reduce all these strong assumptions and provide a nearly optimal bound that exhibits a linear dependence on $d$ up to a logarithmic factor. We are aware that uniformization appeared in the proof of \cite{NEURIPS2022_b5b52876}. However, this work is the first to deploy uniformization as a working algorithm for discrete diffusion models and prove its efficiency.

In addition, some works focus on other aspects of the theoretical understanding of diffusion models. For example, \cite{li2023faster,chen2023probability} analyze the (approximately) deterministic algorithm for the reversed sampling of diffusion model; \cite{li2024generalization,gupta2023sampleefficient} studies the sample complexity of learning the score function.

\section{Preliminaries on Continuous-time Markov Chain}\label{sc2}
Let $\cX$ be a finite state space with $|\cX| = N$. A CTMC $(X_t)_{t \ge 0} $ on $\cX$ is a continuous stochastic process that satisfies the Markov property. In this process, state changes occur based on an exponential clock associated with each state, with transitions to different states determined by the corresponding transition rates.

Mathematically, a continuous-time markov chain $(X_t)_{t\ge 0}$ on $\cX$ is characterized by its infinitesimal generator $Q(t) \in \RR^{N\times N}$ with $Q_{x,x}(t) = -\sum_{y \ne x}Q_{x,y}(t)$. Denote the transition kernel of $(X_t)$ from time $s$ to time $t$, i.e., $\PP(X_t = y|X_s = x) = P_{x,y}(s,t)$. The infinitesimal transitions of the process are determined by the generator through $P(t,t+h) = I+hQ(t)+o(h)$. This leads to the Kolmogorov forward equation $\frac{\d }{\d t}{P(s,t)} = P(s,t)Q(t). $ In particular, For time-homogeneous cases $Q(t) \equiv Q$, we have $P(s,t)  = e^{(t-s)Q}$.

 To understand how the process $(X_t)$ can be constructed, let's start with the simple case that $Q$ is a time-homogeneous generator $Q(t)  \equiv Q$. For each $x, y \in \cX$,  $Q_{x,y} \ge 0$ can be viewed as the transition rate from the state $x$ to the state $y$ and $-Q_{x,x} = \sum_y Q_{x,y}$ is the intensity of the exponential clock for the state $x$. Starting from a state $x \in \cX$, the process transitions after a holding time determined by an $\mathrm{Exp}(Q_{x,x})$ random variable, with the transition probabilities defined as
$$ \PP(t = y| t_-= x ) = -Q_{x,y}/Q_{x,x},\quad \forall y \ne x.  $$

The time-inhomogeneous case is slightly more complicated. Intuitively, it can be understood as a time-homogeneous CTMC with a generator $Q(t)$ at each infinitesimal interval $[t,t+dt]$. So, a natural approach for numerical simulation involves discretizing the process and simulating the time-homogeneous CTMC within each distinct time interval. However, more sophisticated methods exist for simulating a time-inhomogeneous CTMC without the need for discretization. This method is known as \emph{uniformization}. It decouples the clocks and transition mechanisms to a single Poisson point process, and a set of Markov transition kernels $(\tilde{P}(t))_{t\ge 0}$. The intensity of the Poisson point process $\lambda$ uniformly bounds all clock intensities $Q_{xx}$, and the transition kernels are defined by:
\begin{align} \label{kernel}   \tilde{P}_{x,y}(t) =
\begin{cases}
\frac{{Q}_{x,y}(t)}{\lambda}    &   y \ne x \\
1- \sum_{y \ne x}\frac{{Q}_{x,y}(t)}{\lambda} & y  = x
\end{cases},
\end{align}
 or in a matrix representation $\tilde{P}(t) = I + \frac{1}{\lambda}Q(t)$.
 Simulating the CTMC involves changing the state according to $\tilde{P}(t)$ whenever a transition occurs in the Poisson point process. Formally, we have the following proposition:
\begin{proposition}[Uniformization of CTMC]\label{uniformization}
Consider a general CTMC on a finite state space $\cX$ with the generator $Q(t)$. Let $p(t)$ be the distribution of the CTMC at time $t$. Suppose $Q_{x,x}(t) \le \lambda$ for any $x \in \cX$ and $0\le t \le T$. Let $(\tilde{P}(t))_{t \ge 0}$ be the transition kernels given by \eqref{kernel}. Let $\tau_1 < \tau_2 < \cdots < \tau_n$ be the transition times within $[0,T]$ of a Poisson point process with intensity $\lambda$, or equivalently, $n$ is drawn from $\mathrm{Poisson}(\lambda)$ and $\tau_1,\cdots,\tau_n$ is sorted i.i.d. samples from $\mathrm{Unif}([0,T])$. Conditioning on the number of transition $n$ by time $T$ and on the transition times $\tau_1,\tau_2,\cdots,\tau_n $, we have $p(T)|(n,\tau_1,\cdots,\tau_n) = p(0)\tilde{P}(\tau_1)\tilde{P}(\tau_2)\cdots \tilde{P}(\tau_n)$.
\end{proposition}
In the time-homogeneous setting $Q(t) \equiv Q$, Proposition \ref{uniformization} can be simply deduced through a Taylor expansion:
$$ p(t) = p(0)e^Q =p(0) \sum_{n=0}^\infty \tilde{P}^n \frac{(\lambda t)^n}{n!} e^{-\lambda t},      $$
implying the transition of the CTMC $e^Q$ can be executed by applying the transition $\tilde{P}$ for $\mathrm{Poisson}(\lambda t)$ times. For a general setting, an intuitive way to understand the uniformization involves approximating the generator $Q(t)$ by a piece-wisely constant function. The results of uniformization in the homogeneous setting can be easily adapted to cases with piece-wise constant $Q(t)$ by the Markov property. Since any general function $Q(t)$ can be approximated by piece-wise constant function to arbitrary precision, uniformization thus provides an exact simulation of the time-inhomogeneous CTMC process. A rigorous proof of this can be found in \cite{10.1016/0167-6377(92)90086-I}.

In practice, however, simulating CTMC through uniformization may entail substantial computational costs, as the number of necessary transitions is contingent upon the uniform bound imposed on all the clocks. This renders the computation resource-intensive, especially in scenarios involving stiff problems where $Q_{x,x}(t)$ dramatically changes across different states $x$ and different times $t$. Nevertheless, we will demonstrate that, when applying this method to the discrete diffusion model, we can obtain a provable efficient algorithm by adaptively selecting the upper bound $\lambda$ across different time intervals. 

\section{Framework of Discrerte Diffusion Mode} \label{sc3}


\subsection{General Procedure} \label{sc2.1}
\paragraph*{The forward and reversed CTMC}
Let $p(0)$ be a data distribution on a $\cX$. Consider a forward process $(X_t)_{0\le t \le T}$ defined by a CTMC with a generator $Q(t)$ starting from $X_0 \sim p(0)$. The distribution of the forward process at time $t$ is denoted by $p(t)$.

As an analogy of the reversed SDE used in diffusion model on $\RR^d$ \cite{Anderson1982ReversetimeDE,SGMsong}, we
construct a CTMC $(X_t^{\leftarrow})_{0 \le t \le T}$ as a time reversal of  $(X_t)$, meaning $(X_t^{\leftarrow})_{0 \le t \le T} $ is equivalent to $(X_{T-t})_{0 \le t \le T} $ in distribution. As discussed in \cite{NEURIPS2022_b5b52876,benton2023denoising,lou2023discrete}, this time reversal can be achieved by a CTMC starting from $X_0^{\leftarrow} \sim p(T)$ and governed by the generator 
\begin{align}
Q^{\leftarrow}_{x,y}(t) := Q_{y,x}(t)\frac{p_y(t)}{p_x(t)}. \label{reverse-generator}
\end{align}
One can sample from the data distribution $p(0)$ if it is possible to simulate the reversed process $(X_t^{\leftarrow})$.
However, 
the ratio term  $\frac{p_y(t)}{p_x(t)}$ in the generato \eqref{reverse-generator} is not available. This term is referred to as the \emph{discrete score function}.
The idea is to parameterize the discrete score function within a function class, such as neural networks, and learn it from training data.
\paragraph*{Training objectives for learning the discrete score}
We denote the discrete score function as $c_{x,y}(t):=\frac{p_y(t)}{p_x(t)}$ and $s_{x,y}(t)$ represents an estimation of $c_{x,y}(t)$ used for sampling. Consider the sampling dynamic $(\hat{X}_t^{\leftarrow})_{0\le t\le T}$ that is a CTMC with the generator $\hat{Q}^{\leftarrow}_{x,y}(t) := Q_{y,x}(t) s_{x,y}(t)$, initiating  from a simple distribution $\gamma \approx p(T)$. The following proposition gives an expression of the KL divergence between the true reversed process $({X}_t^{\leftarrow})$ and the sampling dynamic  $(\hat{X}_t^{\leftarrow})$.
\begin{proposition}\label{pathKL}
Let $\PP^\leftarrow$ and $\hat{\PP}^\leftarrow$ be the path measure of $({X}_t^{\leftarrow})$ and $(\hat{X}_t^{\leftarrow})$, respectively. We have
\begin{align}\label{pathKL-eq}
\begin{aligned}
\KL(\PP^\leftarrow \| \hat{\PP}^\leftarrow) = \KL(p(T)\|\gamma) + \int_0^T \EE_{X_t \sim p(t)}  \ell(c_{X_t},s_{X_t}) \d t, 
\end{aligned}
\end{align}
where
\begin{align*}
 \ell(c_{x},s_{x}) = \sum_{y \ne x} Q_{y,x}\left(-c_{x,y}(t)+s_{x,y}(t) + c_{x,y}\log \frac{c_{x,y}(t)}{s_{x,y}(t)} \right).
\end{align*}
\end{proposition}
This is an analogy of the Girsanov theorem for the SDE framework. We defer an intuitive derivation to Appendix \ref{proof-pathKL} and refer to \cite[Section A]{benton2023denoising} for rigorous proof.
Note that the first term in the right-hand side of \eqref{pathKL-eq} is fully characterized by the convergence of the forward process; thus, our focus should primarily be on the second term for learning the score function. 
Omitting the terms that are independent of the score estimator $s$, we obtain the \emph{implicit score entropy} \cite{benton2023denoising,lou2023discrete}:
\begin{align} \label{ISM}
 \cL_{\text{ISE}} =   \int_0^T \EE_{X_t \sim p(t)}\Bigg[ \sum_{Y \ne X_t}\big(Q_{Y,X_t}s_{X_t,Y}  - Q_{X_t,Y} \log s_{Y,X_t}(t)\big)\Bigg]   \d t 
 \end{align}

This objective \eqref{ISM} can be optimized by substituting the expectation with an empirical mean over samples.

For enhanced computational efficiency, a variant of implicit score entropy called the \emph{denoising score entropy} is proposed \cite{benton2023denoising,lou2023discrete}:
\begin{align} \label{DSE}
\cL_{\text{DSE}}  =  \int_0^T \EE_{X_0 \sim p(0),\,X_t \sim p(t)}\Bigg[\sum_{Y\ne X_t}Q_{Y,X_t}\big(s_{X_t,Y}(t)  - \frac{P_{X_0,Y}(0,t)}{P_{X_0,X_t}(0,t)} \log s_{X_t,Y}(t) \big) \Bigg] \d t,      
\end{align}
where $P_{x,y}(s,t)$ is the transition kernel of the forward process $(X_t)$. 
Note that the denoising score entropy includes a term from the transition kernel of the forward process, implying that the design of the forward process should facilitate explicit computation of this kernel. One example is the independent flips on hypercube, which is the main focus of this paper and will be introduced in section \ref{sc2.2}. 

Note that the implicit score entropy and the denoising score entropy are analogous to the implicit score matching objective \cite{Hyvrinen2005EstimationON} and the denoising score matching objective \cite{Vincent2011ACB}, respectively, in the setting of a diffusion model on $\RR^d$.

\subsection{Independent flips on hypercube as forward process} \label{sc2.2}
In this section, and throughout the remainder of this paper, we concentrate on a scenario where the state space is the $d$-dimensional hypercube $\cX= \{0,1\}^d$ and the forward process is given by independent flips on each component. This process is formally described as a CTMC with the generator:
\begin{align} \label{OU-generator}
\begin{aligned}
    Q_{x,y}(t)= 
\begin{cases}
1,\, d(x,y) = 1 \\
-d,\, y = x \\
0,\, \text{otherwise}
\end{cases}
\end{aligned}
\end{align}
where $d(x,y)$ denotes the Hamming distance between $x$ and $y$. This process has several nice properties that are critical for the design and analysis of our algorithm in section \ref{sc4}.
\paragraph*{Explicit formula for the transition kernel}
The special structure of the independent flipping process allows us to write the transition kernel explicitly by the following proposition:
\begin{proposition} \label{explicit-transition}
Let $P_{x,y}(s,t)$ be the transition kernel of the CTMC with the generator $Q$ given in \eqref{OU-generator}. Let $g_w(t)$ be the discrete heat kernel
$$  g_w(t) = \frac{1}{2^d} \prod_{i=1}^d (1+(-1)^{w_i}e^{-2t}) ,\quad \forall w \in \{0,1\}^d.       $$
Then we have
$P_{x,y}(s,t) = g_{y-x}(t-s)$, where the $y-x$ should be understood in modulo 2 sense.
\end{proposition}
\begin{proof}
By utilizing the tensor-product structure of $Q$, we can write the generator $Q$ as
$       Q = \sum_{i=1}^d Q_i,$
where $Q_i$ represents a tensor product of 
$d$ matrices of size $2\times 2$. In this tensor product, the matrix $A :=  \begin{pmatrix} -1 & 1  \\ 1 & -1   \end{pmatrix}$ occupies the $i$-th position, while the remaining positions are all identity. By the Kolmogorov forward equation, we have
\begin{align*} P(s,t) & =  e^{(t-s)Q} =  (e^{(t-s)A})^{\otimes d} \\ & = \begin{pmatrix} \frac{1+e^{-2(t-s)}}{2} & \frac{1-e^{-2(t-s)}}{2}  \\ \frac{1-e^{-2(t-s)}}{2} & \frac{1+e^{-2(t-s)}}{2}   \end{pmatrix}^{\otimes d},
\end{align*}
giving the expression $P_{x,y}(s,t) = g_{y-x}(t-s)$.
\end{proof}
This explicit formula of the transition kernel makes the denoising score entropy \eqref{DSE} tractable and will play an important role in our algorithm proposed in section \ref{sc4} for simulating the reversed process. 
\paragraph*{Convergence of the forward process}
In the context of the diffusion model on $\RR^d$, the OU forward process is known to converge exponentially to standard Gaussian, regardless of the complexity of the data distribution \cite{SamplingEasy,chen2023improved}. Similarly, in our discrete setting, the independent flipping process exhibits an exponential convergence to the uniform distribution over $\{0,1\}^d$. We write this formally in the following proposition:
\begin{proposition} \label{converge-forward}
Let $p(t)$ be the distribution of the forward process with the generator $Q$ given in \eqref{OU-generator} at time $t$. Let $\gamma$ denote the uniform distribution over $\{0,1\}^d$. We have
$$ \KL(p(T) \| \gamma)   \le  e^{-T} \KL(p(0) \| \gamma) \lesssim de^{-T}           $$
\end{proposition}
The proof is deferred to Appendix \ref{proof-convergenceforward}. Proposition \ref{converge-forward} suggests that to make the error from the insufficient mixing of the forward process less than $\epsilon$, i.e., $\KL(p_T\| \gamma) \le \epsilon$, we only need to simulate the forward process for time $T = \log \left(\frac{d}{\epsilon} \right)$, which depends on $d$ and $1/\epsilon$ logarithmly.

\paragraph*{The Sparse Structure}
Since transitions only occur between neighbors on the hypercube each time, we can use a more concise way to represent the score function
$$    c_x(t) =\left[\frac{p_{x+e_1}}{p_x},  \frac{p_{x+e_2}}{p_x},\cdots,  \frac{p_{x+e_d}}{p_x}\right]^\top \in \RR^d.              $$
Here, $e_i$ denotes the vector with a 1 in the $i$-th coordinate and 0's elsewhere, and the addition operator is defined modulo 2. Similarly, we use $s_x(t)$ to denote the score estimator. 

Note that for a state space of size $2^d$, the discrete score function, defined by the ratio of probabilities, is generally a function $\{0,1\}^d \times \{0,1\}^d \to \RR $. However, by leveraging the sparse structure of the transitions, we can simplify this to a $\{0, 1\}^d  \times d \to \RR$ function. This simplification enables more efficient computation when learning the discrete score function. Furthermore, as we will discuss in Section \ref{sc4}, this sparse structure greatly facilitates the algorithmic implementation of the reversed process.

\section{Main Results} \label{sc4} 
In this section, we present our algorithm and the analysis for implementing the sampling dynamic, which is a reversed process corresponding to the independent flipping forward process, with the discrete score function replaced by the score estimator. Because of \eqref{OU-generator}, the sampling dynamic is given by a CTMC with the generator:
\begin{align} \label{sampling-generator}
\begin{aligned}
    \hat{Q}_{x,y}^\leftarrow(t)= 
\begin{cases}
s_{x}(t)_i ,\, y = x+e_i \\
-\sum_{i=1}^d s_{x}(t)_i,\, y = x \\
0,\, \text{otherwise} 
\end{cases}                     .
\end{aligned}
\end{align}           
\subsection{Assumptions}
To begin with, we introduce some assumptions related to the learning of the discrete score function. Firstly, we assume the score estimator is $\epsilon$-accurate in the score entropy loss.
\begin{assumption} \label{as1}
For some $\epsilon,\delta > 0$, the score estimator $s$ satisfies
$$  \cL(s) = \frac{1}{T-\delta} \int_\delta^T \EE_{X_t \sim p(t)} \ell( c_{X_t}(t), s_{X_t}(t)) \d t   \le \epsilon,    $$
where $\ell: \RR^d \times \RR^d \to \RR$ is the Bregman distance w.r.t. the entropy function $h(x) = \sum_{i=1}^d x_i \log x_i$ given by
\begin{align*} \ell(c,s) = c - s - \langle \nabla h(s), s-c \rangle  = \sum_{i=1}^d \left(-c_i+s_i + c_i\log \frac{c_i}{s_i}\right)\end{align*}
\end{assumption}
 Assumption \ref{as1} is an analogy of the $L^2$-accuracy assumption that is widely used in the theoretical works for the diffusion models in the SDE framework \cite{convergence-score, convergencescore2, SamplingEasy, chen2023improved, li2023faster, benton2023linear}. There are, however, two primary distinctions between Assumption \ref{as1} and the corresponding assumption for diffusion models on $\RR^d$.
 First, in our discrete setting, the $L^2$ distance is replaced by the Bregman distance. This choice aligns with the structure of the CTMC framework, where the KL divergence between path measures is characterized by the Bregman distance (See Proposition \ref{pathKL}) rather than the $L^2$ distance for the SDE framework. Second, we assume the score estimation error to be small on average over the time interval $[\delta, T]$, in contrast to the assumption in continuous diffusion models, which assume a small average error over specific discretization points. This difference arises because our algorithm employs the score estimator at randomly sampled times rather than at a fixed set of discretization points. Both assumptions are reasonable because one can sample times either uniformly or from a discrete set during the stochastic gradient descent in the training process.
 
It is important to notice that the term $\cL(s)$ in Assumption \ref{as1} is equivalent to (up to a term independent of $s$) the objective functions used in the learning procedure, including the implicit score entropy and the denoising score entropy discussed in section \ref{sc2.1}. Thus, Assumption \ref{as1} is satisfied if the objective function is optimized to the extent that its function value is $\epsilon$ close to the minimum value.

The second assumption involves the uniform boundedness of the score estimator $s$. This leads to a bounded transition rate in the sampling dynamics, thereby enabling the algorithmic application of the uniformization technique. This assumption relies on the observation that the probability ratio of the forward process (or the true discrete score function) is uniformly bounded.
\begin{proposition} \label{bound-score}
Let $p(t)$ be the distribution of the forward CTMC with generator $Q$ given in \eqref{OU-generator}. For $t>0$, the probability ratio between two neighbor states is bounded by
$$\frac{p_{x+e_i}(t)}{p_x(t)} \lesssim \max\left(\frac{1}{t},1 \right),\quad \forall x \in \{0,1\}^d,\,i \in \{1,\cdots,d\}.$$
\end{proposition}
\begin{proof}
Let $g$ be the heat kernel defined in Proposition \ref{explicit-transition}. Consider the conditional distribution $\tilde{p}(a|x) \propto p_a(0)g_{x-a}(t)$ that is the distribution of $X_t$ conditional on $X_0$. Then we write the probability ratio as (in the following computation, all the addition is defined modulo 2):
\begin{align}
 \frac{p_{x+e_i}(t)}{p_x(t)} & = \frac{\sum_{a+w=x+e_i}p_a(0)g_w(t) }{\sum_{a+w=x}p_a(0)g_w(t)} \notag \\
 &  = \frac{\sum_{a+w=x}p_a(0)g_{w+e_i}(t) }{\sum_{a+w=x}p_a(0)g_w(t)} \notag \\
 &  =  \frac{\sum_{a+w=x}p_a(0)\frac{g_{w+e_i}(t)}{g_w(t)}g_w(t) }{\sum_{a+w=x}p_a(0)g_w(t)} \label{000}
\end{align}
By the definition of $g$, we can compute the ratio  $\frac{g_{w+e_i}(t)}{g_w(t)} = \frac{1-(-1)^{w_i}e^{-2t}}{1+(-1)^{w_i}e^{-2t}}$. Substitute this to \eqref{000}, we obtain
\begin{align*}
\frac{p_{x+e_i}(t)}{p_x(t)}  &  = \frac{\sum_{a+w=x}p_a(0) \frac{1-(-1)^{w_i}e^{-2t}}{1+(-1)^{w_i}e^{-2t}} g_w(t) }{\sum_{a+w=x}p_a(0)g_w(t)} \\
 & = \frac{\sum_{a}p_a(0) \frac{1-(-1)^{x_i-a_i}e^{-2t}}{1+(-1)^{x_i-a_i}e^{-2t}} g_{x-a}(t) }{\sum_{a}p_a(0)g_{x-a}(t)} \\
 & =  \EE_{\tilde{p}(a|x)}\left[\frac{1-(-1)^{x_i-a_i}e^{-2t}}{1+(-1)^{x_i-a_i}e^{-2t}} \right].\\
 & \lesssim \frac{1+e^{-2t}}{1-e^{-2t}}
\end{align*}
 Since $\frac{1+e^{-2t}}{1-e^{-2t}} \lesssim \max\left(1,\frac{1}{t}\right)$, we complete the proof.
\end{proof}
Note the bound in Proposition \ref{bound-score} is independent of the data distribution; thus, we do not need any assumption on the structure of data distribution. 

Based on Proposition \ref{bound-score}, it is natural to assume the score estimator is also uniformly bounded. 
\begin{assumption}\label{as2}
We assume there is a universal constant $C$ such that
$$ \sum_{i=1}^d s_x(t)_i \le Cd\max\left(1,\frac{1}{t}\right).  $$
\end{assumption}
In practice, the constraint in Assumption \ref{as2} can be imposed by 
slightly modifying the learned score function after training. For example, one can add a sigmoid layer before the output for the score estimator to satisfy the desired bound. This modification will not affect the accuracy of the learned score function according to Proposition \ref{bound-score}.
\subsection{Algorithm}
We provide an algorithm, detailed in Algorithm \ref{alg} that exactly simulates the sampling dynamic. Our algorithm is based on the uniformization technique stated in Proposition \ref{uniformization}. Note Proposition \ref{uniformization} requires the uniform boundness condition for the transition rates, which is not satisfied by the generator $\hat{Q}^\leftarrow$ in general. To address this, our algorithm implements early stopping with a terminal time of 
$T-\delta$. Then Assumption \ref{as2} ensures the bound. Moreover, since the bound in Assumption \ref{as2} varies over time, we apply the procedure from Proposition \ref{uniformization} with adaptive $\lambda$. Specifically, we introduce a partition $0 = t_0 < t_1 < \cdots < t_N =T-\delta$ and set different $\lambda_k$ values in different time intervals. Combining all these ingredients leads to the formulation of Algorithm \ref{alg}. The algorithm outputs a sample from the distribution $\hat{p}^\leftarrow(T-\delta)$, where $\hat{p}^\leftarrow(t)$ represents the distribution of the CTMC with the generator defined in \eqref{sampling-generator} at time $t$.
\begin{algorithm}[tb]
   \caption{Exact algorithm for implementing the sampling dynamic}
   \label{alg}
\begin{algorithmic}
   \STATE {\bfseries Input:} A learned score function $s$ that satisfies Assumption \ref{as2} with a constant $C$, total time $T$, a time partition $0 = t_0 < t_1 < \cdots < t_N = T-\delta$, parameters $\lambda_1,\lambda_2,\cdots,\lambda_N $.
   \STATE {\bfseries Output:} A sample from $\hat{p}^\leftarrow(T-\delta)$.
   \STATE Draw $Y_0 \sim \mathrm{Unif}(\{0,1\}^d)$.
   \FOR{$k=0$ {\bfseries to} $N-1$}
   \STATE Draw $M \sim \mathrm{Poisson}(\lambda_{k+1}(t_{k+1}-t_k))$
   \STATE Sample $M$ points i.i.d. from $\mathrm{Unif}([t_{k},t_{k+1}])$ and sort them as $\tau_1 < \tau_2 \cdots < \tau_{M}$
   \STATE Set $Z_0 = Y_k$
   \STATE \FOR{$j=0$ {\bfseries to} $M-1$}
   \STATE \STATE Set $Z_{j+1} = \begin{cases} 
   Z_{j}+e_i,\, & \text{w.p.}\, \frac{s_{Z_{j}}(\tau_j)_i}{\lambda_{k+1}},\,1\le i \le d \\
    Z_{j},\,    & \text{w.p.}\, 1- \sum_{i=1}^d \frac{s_{Z_{j}}(\tau_j)_i}{\lambda_{k+1}}
    \end{cases}$  
   \STATE \ENDFOR
   \STATE Set $ Y_{k+1}=Z_M    $
   \ENDFOR
   \STATE {\bfseries return} $Y_N$
\end{algorithmic}
\end{algorithm}

\subsection{Analysis of the Algorithm}
In this section, we provide a theoretical analysis of Algorithm \ref{alg} on its complexity and error. In particular,  we are interested in the following two aspects of Algorithm \ref{alg}:
\begin{itemize}
\item The algorithm complexity. In other words, how many steps are required to implement Algorithm \ref{alg}?
\item The error of the resulting distribution, i.e., the distance between the data distribution $p(0)$ and the distribution $\hat{p}^\leftarrow(T-\delta)$ obtained by Algorithm \ref{alg}
\end{itemize}
\paragraph*{Results for general data distribution}
Firstly, we consider the setting that $p(0)$ is a general data distribution.
We summarize the result in the following theorem:
\begin{theorem}\label{general-bound}
Suppose Assumption \ref{as2} holds. By choosing the time partition such that \begin{align}\label{time-partition} t_{k+1} - t_k \le c(T-t_{k+1}),\, \forall 0\le k \le N-1\end{align} for some absolute constant $c$ and choosing $\lambda_k$'s by 
$$\lambda_k =  \frac{Cd}{\min(1,T-t_{k})},$$
the implementation of Algorithm \ref{alg} requires $M \sim \mathrm{Poisson}(\lambda)$ steps with $\lambda = O(d(\log(1/\delta)+T))$. Moreover, if we further assume Assumption \ref{as1}, Algorithm \ref{alg} outputs a distribution $\hat{p}^\leftarrow(T-\delta)$ such that
\begin{enumerate}
\item The KL divergence between $\hat{p}^\leftarrow(T-\delta)$ and $p(\delta)$ is bounded by
$$  \KL(p(\delta) \| \hat{p}^\leftarrow(T-\delta)) \lesssim de^{-T} + T\epsilon,            $$
Choosing $T \asymp \log(d/\epsilon) $ makes this $\tilde{O}(\epsilon)$ and $\lambda = O(d\log(d/\epsilon \delta))$.
\item \label{TVbound} The TV distance between $\hat{p}^\leftarrow(T-\delta)$ and the data distribution is bounded by
$$\TV(p(0), \hat{p}^\leftarrow(T-\delta)))  \lesssim \sqrt{de^{-T} + T\epsilon} + (1-e^{-d\delta}).           $$
Choosing $T \asymp \log(d/\epsilon),\,\delta \asymp \frac{\sqrt{\epsilon}}{d}$ makes the error $\tilde{O}(\sqrt{\epsilon})$, and $\lambda = O(d\log(d/ \epsilon^{3/4})) $.
\end{enumerate}
\end{theorem}
\begin{remark}
Note that Poisson distribution satisfies super-exponential decaying tails, so $O(d\log(1/\lambda))$ steps are enough with high probability and $\lambda = O(d\log(d/\epsilon \delta))$ or $\lambda = O(d\log(d/ \epsilon^{3/4})) $ describes the complexity of the algorithm. 
\end{remark}
\begin{remark}
The choice of the time partition \eqref{time-partition} is quite flexible. This is because we only care about the total number of transitions rather than the number of intervals $N$
\end{remark}
The proof is deferred to Appendix \ref{proof-generalbound}
Our nearly linear dependence on $d$ and logarithm dependence on $\delta$ match the best result for the continuous diffusion model \cite{benton2023linear}.
\paragraph*{Results for data distribution with uniformly bounded score}
In addition, if we further assume the discrete score function of the data distribution(i.e., the probability ratio between neighbor states) is uniformly bounded, we can simulate the sampling dynamic without early stopping(i.e.,$\delta = 0$). In this case, we can improve the TV distance bound between the data distribution and the sampled distribution in Theorem \ref{general-bound}(\ref{TVbound}) to a KL divergence bound.
\begin{assumption} \label{as3}
Suppose the data distribution $p(0)$ satisfies  
          $$\frac{p_{x+e_i}(0)}{p_x(0)} \leq L,\quad \forall x \in \{0,1\}^d,\,i \in \{1,\cdots,d\}.$$
\end{assumption}
\begin{theorem} \label{bounded-ratio}
Under Assumption \ref{as1},\ref{as2},\ref{as3}, let $\delta = 0,\,T \asymp \log(d/\epsilon)$, by choosing the time partition $ 0 = t_1 < \cdots < t_N = T$ and parameter $\lambda_k$'s appropriately in algorithm \ref{as1}, one can obtain a distribution $\hat{p}^{\leftarrow}(T)$ that is $\tilde{O}(\epsilon)$ close to the data distribution within $\mathrm{Poisson}(O(d\log(dL/\epsilon)))$ steps.
\end{theorem}
The proof is deferred to Appendix \ref{proof-boundedratio}. Notably, the algorithm complexity only has a logarithm dependence on $L$, so we can get a reasonable guarantee even if $L$ has an exponential dependence on $d$.  
 \paragraph*{Discussion on the near optimality}
Note that for the true reversed process, the expectation value of the intensity of each clock is given by
\begin{align*}
\EE \sum_{i=1}^d \frac{p_{x+e_i}(t)}{p_x(t)} &= \sum_{x \in \{0,1\}^d} \sum_{i=1}^d p_{x+e_i}(t) \\
& =  \sum_{i=1}^d \sum_{x\in \{0,1\}^d} p_{x+e_i}(t) = d
\end{align*}
Consequently, the expected number of transitions within a time interval of length 1 is $\Theta(d)$. As a result, a linear dependence on $d$ is unavoidable for simulating the reversed process. 
\paragraph*{Comparison with the $\{0,1,\cdots,n\}^d$ setting} \cite{NEURIPS2022_b5b52876,lou2023discrete} considered the forward process on $\{1,\cdots,n\}^d$ with the generator $Q$ defined as follows: for two distinct states $x, y \in \{0, 1, \cdots, n\}^d$, define $Q_{x, y} = 1$ if and only if $x$ and $y$ differ in exactly one entry. In this case, the typical number of transitions that occur within a unit time interval is $\Theta(n^2 d)$ and the convergence of the forward process requires at least $\Omega(1)$ time. So $\Omega(n^2d)$ steps are required to simulate the reversed process. However, if we transform the data to the $\{0,1\}^{d\log(n)}$ structure and utilize our hypercube framework, the $n,d$ dependence of the algorithm complexity is reduced to $\tilde{O}(d\log(n)) $. Therefore, our hypercube framework provides a more efficient implementation of the discrete diffusion model.
\section{Conclusion}
In this paper, we consider an algorithm for discrete diffusion models based on uniformization and present the first theoretical analysis of the discrete diffusion model. Although our nearly linear dependence result aligns with the state-of-art result for diffusion models on $\RR^d$ and is nearly optimal in the current framework, there are several interesting further directions to explore:
\paragraph*{Faster algorithm with theoretical guarantee}
Our algorithm provides an exact simulation of the reversed process, where the number of transitions corresponds to the worst-case bound of the clock intensities. Although we believe that the $\Omega(d)$ complexity is not improvable in general, there may be potential in investigating an approach that simulates transitions adaptive to the clock of the current state. This might require some discretization and further analysis to quantify the discretization error.
\paragraph*{Improve the graph structure of the forward process} We consider the independent flipping process on the hypercube as the forward process. This process converges to the uniform distribution in $O(\log d)$ time and results in a reversed process that transitions $\Omega(d)$ times. A natural question is if one can employ a better structure for the design of the forward process so that the forward process still converges exponentially but the number of transition times is reduced. One idea is to apply the forward process on the Ramanujan graph \cite{Lubotzky2017RamanujanG}, but the numerical simulation will become hard.


\bibliographystyle{amsalpha}
\bibliography{ref}

\newcommand{\etalchar}[1]{$^{#1}$}
\providecommand{\bysame}{\leavevmode\hbox to3em{\hrulefill}\thinspace}
\providecommand{\MR}{\relax\ifhmode\unskip\space\fi MR }
\providecommand{\MRhref}[2]{%
  \href{http://www.ams.org/mathscinet-getitem?mr=#1}{#2}
}
\providecommand{\href}[2]{#2}
\begin{thebibliography}{SDWMG15}

\bibitem[AJH{\etalchar{+}}23]{austin2023structured}
Jacob Austin, Daniel~D. Johnson, Jonathan Ho, Daniel Tarlow, and Rianne van~den Berg, \emph{Structured denoising diffusion models in discrete state-spaces}, 2023.

\bibitem[And82]{Anderson1982ReversetimeDE}
Brian. D.~O. Anderson, \emph{Reverse-time diffusion equation models}, Stochastic Processes and their Applications \textbf{12} (1982), 313--326.

\bibitem[BBDD23]{benton2023linear}
Joe Benton, Valentin~De Bortoli, Arnaud Doucet, and George Deligiannidis, \emph{Linear convergence bounds for diffusion models via stochastic localization}, 2023.

\bibitem[BLM13]{concentration-book}
Stéphane Boucheron, Gábor Lugosi, and Pascal Massart, \emph{{Concentration Inequalities: A Nonasymptotic Theory of Independence}}, Oxford University Press, 02 2013.

\bibitem[BSB{\etalchar{+}}23]{benton2023denoising}
Joe Benton, Yuyang Shi, Valentin~De Bortoli, George Deligiannidis, and Arnaud Doucet, \emph{From denoising diffusions to denoising markov models}, 2023.

\bibitem[BTYLD23]{bartal2023multidiffusion}
Omer Bar-Tal, Lior Yariv, Yaron Lipman, and Tali Dekel, \emph{Multidiffusion: Fusing diffusion paths for controlled image generation}, 2023.

\bibitem[CBDB{\etalchar{+}}22]{NEURIPS2022_b5b52876}
Andrew Campbell, Joe Benton, Valentin De~Bortoli, Thomas Rainforth, George Deligiannidis, and Arnaud Doucet, \emph{A continuous time framework for discrete denoising models}, Advances in Neural Information Processing Systems (S.~Koyejo, S.~Mohamed, A.~Agarwal, D.~Belgrave, K.~Cho, and A.~Oh, eds.), vol.~35, Curran Associates, Inc., 2022, pp.~28266--28279.

\bibitem[CCL{\etalchar{+}}22]{SamplingEasy}
Sitan Chen, Sinho Chewi, Jerry Li, Yuanzhi Li, Adil Salim, and Anru~R. Zhang, \emph{Sampling is as easy as learning the score: theory for diffusion models with minimal data assumptions}, 2022.

\bibitem[CCL{\etalchar{+}}23]{chen2023probability}
Sitan Chen, Sinho Chewi, Holden Lee, Yuanzhi Li, Jianfeng Lu, and Adil Salim, \emph{The probability flow ode is provably fast}, 2023.

\bibitem[Che23]{chewi2023log}
Sinho Chewi, \emph{Log-concave sampling}, Book draft available at https://chewisinho. github. io (2023).

\bibitem[CLL23]{chen2023improved}
Hongrui Chen, Holden Lee, and Jianfeng Lu, \emph{Improved analysis of score-based generative modeling: User-friendly bounds under minimal smoothness assumptions}, 2023.

\bibitem[DeB22]{Bortoli2022ConvergenceOD}
Valentin DeBortoli, \emph{Convergence of denoising diffusion models under the manifold hypothesis}, 2022, arXiv:2208.05314.

\bibitem[dSeSG00]{deSouzaeSilva2000}
Edmundo de~Souza~e Silva and H.~Richard Gail, \emph{Transient solutions for markov chains}, pp.~43--79, Springer US, Boston, MA, 2000.

\bibitem[GLW{\etalchar{+}}23]{guo2023diffusion}
Zhiye Guo, Jian Liu, Yanli Wang, Mengrui Chen, Duolin Wang, Dong Xu, and Jianlin Cheng, \emph{Diffusion models in bioinformatics: A new wave of deep learning revolution in action}, 2023.

\bibitem[GPPX23]{gupta2023sampleefficient}
Shivam Gupta, Aditya Parulekar, Eric Price, and Zhiyang Xun, \emph{Sample-efficient training for diffusion}, 2023.

\bibitem[Gra77]{GRASSMANN197747}
W.K. Grassmann, \emph{Transient solutions in markovian queueing systems}, Computers \& Operations Research \textbf{4} (1977), no.~1, 47--53.

\bibitem[HNJ{\etalchar{+}}21]{hoogeboom2021argmax}
Emiel Hoogeboom, Didrik Nielsen, Priyank Jaini, Patrick Forré, and Max Welling, \emph{Argmax flows and multinomial diffusion: Learning categorical distributions}, 2021.

\bibitem[HSDL23]{huang2023conditional}
Han Huang, Leilei Sun, Bowen Du, and Weifeng Lv, \emph{Conditional diffusion based on discrete graph structures for molecular graph generation}, 2023.

\bibitem[HSG{\etalchar{+}}22]{ho2022video}
Jonathan Ho, Tim Salimans, Alexey Gritsenko, William Chan, Mohammad Norouzi, and David~J. Fleet, \emph{Video diffusion models}, 2022.

\bibitem[Hyv05]{Hyvrinen2005EstimationON}
Aapo Hyv{\"a}rinen, \emph{Estimation of non-normalized statistical models by score matching}, J. Mach. Learn. Res. \textbf{6} (2005), 695--709.

\bibitem[LLT22a]{convergence-score}
Holden Lee, Jianfeng Lu, and Yixin Tan, \emph{Convergence for score-based generative modeling with polynomial complexity}, 2022.

\bibitem[LLT22b]{convergencescore2}
\bysame, \emph{Convergence of score-based generative modeling for general data distributions}, 2022.

\bibitem[LLZB24]{li2024generalization}
Puheng Li, Zhong Li, Huishuai Zhang, and Jiang Bian, \emph{On the generalization properties of diffusion models}, 2024.

\bibitem[LME23]{lou2023discrete}
Aaron Lou, Chenlin Meng, and Stefano Ermon, \emph{Discrete diffusion language modeling by estimating the ratios of the data distribution}, 2023.

\bibitem[LPS17]{Lubotzky2017RamanujanG}
Alexander Lubotzky, Ralph Phillips, and Peter Sarnak, \emph{Ramanujan graphs}, Combinatorica \textbf{8} (2017), 261--277.

\bibitem[LWCC23]{li2023faster}
Gen Li, Yuting Wei, Yuxin Chen, and Yuejie Chi, \emph{Towards faster non-asymptotic convergence for diffusion-based generative models}, 2023.

\bibitem[MCSE23]{meng2023concrete}
Chenlin Meng, Kristy Choi, Jiaming Song, and Stefano Ermon, \emph{Concrete score matching: Generalized score matching for discrete data}, 2023.

\bibitem[ND21]{nichol2021improved}
Alex Nichol and Prafulla Dhariwal, \emph{Improved denoising diffusion probabilistic models}, 2021.

\bibitem[NSS{\etalchar{+}}20]{niu2020permutation}
Chenhao Niu, Yang Song, Jiaming Song, Shengjia Zhao, Aditya Grover, and Stefano Ermon, \emph{Permutation invariant graph generation via score-based generative modeling}, 2020.

\bibitem[RDN{\etalchar{+}}]{ramesh2022hierarchical}
Aditya Ramesh, Prafulla Dhariwal, Alex Nichol, Casey Chu, and Mark Chen, \emph{Hierarchical text-conditional image generation with clip latents}.

\bibitem[RPG{\etalchar{+}}21]{ramesh2021zeroshot}
Aditya Ramesh, Mikhail Pavlov, Gabriel Goh, Scott Gray, Chelsea Voss, Alec Radford, Mark Chen, and Ilya Sutskever, \emph{Zero-shot text-to-image generation}, 2021.

\bibitem[Sch23]{schneider2023archisound}
Flavio Schneider, \emph{Archisound: Audio generation with diffusion}, 2023.

\bibitem[SDME21]{song2021maximum}
Yang Song, Conor Durkan, Iain Murray, and Stefano Ermon, \emph{Maximum likelihood training of score-based diffusion models}, 2021.

\bibitem[SDWMG15]{pmlr-v37-sohl-dickstein15}
Jascha Sohl-Dickstein, Eric Weiss, Niru Maheswaranathan, and Surya Ganguli, \emph{Deep unsupervised learning using nonequilibrium thermodynamics}, Proceedings of the 32nd International Conference on Machine Learning (Lille, France) (Francis Bach and David Blei, eds.), Proceedings of Machine Learning Research, vol.~37, PMLR, 07--09 Jul 2015, pp.~2256--2265.

\bibitem[SE19]{SGMsong}
Yang Song and Stefano Ermon, \emph{Generative modeling by estimating gradients of the data distribution}, Advances in Neural Information Processing Systems, vol.~32, 2019.

\bibitem[SFLL23]{santos2023blackout}
Javier~E Santos, Zachary~R. Fox, Nicholas Lubbers, and Yen~Ting Lin, \emph{Blackout diffusion: Generative diffusion models in discrete-state spaces}, 2023.

\bibitem[SSK{\etalchar{+}}20]{SGMSDEsong}
Yang Song, Jascha Sohl{-}Dickstein, Diederik~P. Kingma, Abhishek Kumar, and Ben Poole, \emph{Score-based generative modeling through stochastic differential equations}, International Conference on Learning Representations, 2020.

\bibitem[SYD{\etalchar{+}}23]{sun2023scorebased}
Haoran Sun, Lijun Yu, Bo~Dai, Dale Schuurmans, and Hanjun Dai, \emph{Score-based continuous-time discrete diffusion models}, 2023.

\bibitem[SZD{\etalchar{+}}19]{seff2019discrete}
Ari Seff, Wenda Zhou, Farhan Damani, Abigail Doyle, and Ryan~P. Adams, \emph{Discrete object generation with reversible inductive construction}, 2019.

\bibitem[{van}92]{VANDIJK1992339}
Nico~M. {van Dijk}, \emph{Approximate uniformization for continuous-time markov chains with an application to performability analysis}, Stochastic Processes and their Applications \textbf{40} (1992), no.~2, 339--357.

\bibitem[VD92]{10.1016/0167-6377(92)90086-I}
Nico~M. Van~Dijk, \emph{Uniformization for nonhomogeneous markov chains}, Oper. Res. Lett. \textbf{12} (1992), no.~5, 283–291.

\bibitem[Vin11]{Vincent2011ACB}
Pascal Vincent, \emph{A connection between score matching and denoising autoencoders}, Neural Computation \textbf{23} (2011), 1661--1674.

\bibitem[YSM22]{yang2022diffusion}
Ruihan Yang, Prakhar Srivastava, and Stephan Mandt, \emph{Diffusion probabilistic modeling for video generation}, 2022.

\bibitem[ZYYK23]{zheng2023reparameterized}
Lin Zheng, Jianbo Yuan, Lei Yu, and Lingpeng Kong, \emph{A reparameterized discrete diffusion model for text generation}, 2023.

\end{thebibliography}

\newpage
\appendix
\onecolumn
\section{Omitted Proof}
\subsection{An intuitive proof of Proposition \ref{pathKL}}  \label{proof-pathKL}
Let $\tilde{X}_t^{\leftarrow}$ be a CTMC with the estimated generator $\hat{Q}^\leftarrow$ starting from $\tilde{X}_0^{\leftarrow} \sim p(T)$. Let $\tilde{\PP}^\leftarrow$ be the path measure of $(\tilde{X}_t^{\leftarrow}) $. By the chain rule of KL divergence
$$\KL(\PP^\leftarrow \|\hat{\PP}^\leftarrow ) = \KL(p(T) \| \gamma) + \KL(\PP^\leftarrow \| \tilde{\PP}^\leftarrow).  $$
Now we give an intuitive computation of $\KL(\PP^\leftarrow \| \tilde{\PP}^\leftarrow)$. Consider a path $\gamma$ from time $t=0$ to $T$ (corresponding to the time of the reversed process $(X_t^\leftarrow)_{0\le t \le T}$). Let us first compute the probability ratio $\frac{\PP^\leftarrow(\gamma)}{\tilde{\PP}^\leftarrow(\gamma)}$. We discretize $\gamma(t)$ with step size $\epsilon$ as $x_0,x_1,\ldots,x_L$, where $x_j=\gamma(\epsilon j)$ and $L=T/\epsilon$. $\frac{\PP^\leftarrow(\gamma)}{\tilde{\PP}^\leftarrow(\gamma)}$ is a product of multiple terms $\prod_{i} \frac{\PP^\leftarrow(x_{i+1}|x_i)}{\tilde{\PP}^\leftarrow(x_{i+1}|x_i)}$
\begin{itemize}
\item When $x_i =x_{i+1}$, the ratio is $\frac{1-\sum_{y\ne x_i} Q_{y,x_i}c_{x_i,y}\epsilon}{1-\sum_{y \ne x_i} Q_{y,x_i}s_{x_i,y}\epsilon} $.
\item When $x_i\not=x_{i+1}$, the ratio is $\frac{Q_{x_{i+1}, x_{i}} c_{x_{i}, x_{i+1}}\epsilon}{Q_{x_{i+1}, x_{i}}s_{x_{i},x_{i+1}}\epsilon} = \frac{c_{x_i,x_{i+1}}}{s_{x_i,x_{i+1}}}$.
\end{itemize}
Taking the product and log gives
\[
\log  \frac{\d \PP^\leftarrow(\gamma)}{\d  \tilde{\PP}^\leftarrow(\gamma)} \approx
\sum_{i:\text{no jump}} \left(\sum_{y\ne x_i} Q_{y,x_i} (-c_{x_i, y} + s_{x_i, y}) \right) \epsilon +
\sum_{i:\text{jump}} \sum_{y\ne x_i} \log \frac{c_{x_i,y}}{s_{x_i,y}} \delta_{x_i \rightarrow y}
\]
where $\delta_{x_i \rightarrow y}$ is equal to $1$ if $x_i$ jumps to $y$ otherwise $0$.

The KL divergence is then the expectation of this quantity w.r.t. $\PP^\leftarrow$. To simplify, we use the fact that $\EE \delta_{x_i \rightarrow y} = Q_{y,x_i} c_{x_i,y}\epsilon$ for $y\ne x_i$ and take the limit $\epsilon \to 0$ to obtain
\begin{align*}
&  \EE_{\PP^\leftarrow} \log  \frac{\d \PP^\leftarrow(\gamma)}{\d  \tilde{\PP}^\leftarrow(\gamma)} \\ & =  \EE_{\mathbb{P}^\leftarrow}\left[ \int_0^T \sum_{y \ne X_t^\leftarrow} Q_{y,X_t^\leftarrow}(T-t) \left(- c_{X_t^\leftarrow,y}(T-t) + s_{X_t^\leftarrow,y}(T-t)  + c_{X_t^\leftarrow,y}\log \frac{c_{X_t^\leftarrow,y}(T-t)}{s_{X_t^\leftarrow,y}(T-t)} \right)\d t \right]. 
\end{align*}
Finally, by reversing the time, we have
\begin{align*}
  \EE_{\PP^\leftarrow} \log  \frac{\d \PP^\leftarrow(\gamma)}{\d  \tilde{\PP}^\leftarrow(\gamma)}  =  \EE_{\mathbb{P}}\left[ \int_0^T \sum_{y \ne X_t} Q_{y,X_t}(t) \left(- c_{X_t,y}(t) + s_{X_t,y}(t)  + c_{X_t,y}\log \frac{c_{X_t,y}(t)}{s_{X_t,y}(t)} \right)\d t \right], 
\end{align*}
which completes the proof.
\subsection{Proof of Proposition \ref{converge-forward}} \label{proof-convergenceforward}
\begin{proof}
For the first inequality, it is shown in \cite[Theorem 5.1]{concentration-book} that the uniform distribution over the hypercube satisfies the log-Sobolev inequality with constant 2 (w.r.t. the Markov semigroup associated with the generator $Q$).
 This implies the exponential mixing of the forward process in KL divergence (see, for example, \cite{chewi2023log}):
$$       \KL(p(T) \| \gamma)   \le  e^{-T} \KL(p(0) \| \gamma).      $$

The second inequality is because the KL divergence between $p_0$ and the uniform distribution can be decomposed to 
\begin{align}
\KL(p_0 \| \gamma) \notag  & = \sum_{x \in \{0,1\}^d} p_x(0) \log p_x(0) -  \sum_{x \in \{0,1\}^d} p_x(0) \log \frac{1}{2^d} \notag \\
  & \lesssim  \sum_{x \in \{0,1\}^d} p_x(0) \log p_x(0) + d \label{entropy}
\end{align}
For distributions on a finite set, the maximum entropy is achieved by the uniform distribution, so the entropy term in \eqref{entropy} could be bounded by the entropy of the uniform distribution, which is $d$ (up to a constant).
\end{proof}
\subsection{Proof of Theorem \ref{general-bound} }    \label{proof-generalbound}  
\paragraph*{Proof of the bound of $\lambda$}
In algorithm \ref{alg}, since the number of steps in each time interval $[t_k,t_{k+1}]$ is sampled from $M_k \sim \mathrm{Poisson}(\lambda_k(t_{k}-t_{k-1}))$, the total number of steps follows a Poisson distribution with parameter $\lambda: = \sum_{k=1}^N \lambda_k(t_k-t_{k-1})$. Now we evaluate $\mu$. Recall that 
$$ \lambda_k = \frac{Cd}{\min(1,T-t_k)}.$$
Let $s_{k} = T-t_k$. Recall that 
$$ \lambda_k = \frac{Cd}{\min(1,s_k)},\quad     s_k - s_{k+1} \le cs_{k+1}.   $$
For $\delta \le s_k \le 1$ we have
\begin{align*}
\sum_{k: \delta \le s_k < 1} \lambda_k(s_{k-1}-s_k) = \sum_{k: s_k < 1} \frac{Cd}{s_k}(s_{k-1} - s_{k}) \lesssim  \sum_{k: s_k < 1} \frac{C(c+1)d}{s_{k-1}}(s_{k-1} - s_{k}) \lesssim \int_\delta^1 \frac{d}{s} \mathrm{d} s = d \log \left(\frac{1}{\delta}\right). 
\end{align*}
For $s_k \ge 1$, we have
$$ \sum_{k: 1 \le s_k \le T} \lambda_k(s_{k-1}-s_k) =\sum_{k: 1 \le s_k \le T} Cd (s_k - s_{k-1})= Cd(T-1)  \lesssim dT.   $$
Combining the two parts, we conclude that 
$$ \lambda = \sum_{k=1}^N \lambda_k(t_k - t_{k-1}) = \sum_{k=1}^N \lambda_k (s_k - s_{k-1}) = \sum_{k: \delta \le s_k < 1} \lambda_k(s_{k-1}-s_k) + \sum_{k: 1 \le s_k \le T} \lambda_k(s_{k-1}-s_k) \lesssim d\left(T+\log\left(\frac{1}{\delta} \right)\right). $$  \qed
\paragraph*{Proof of the KL divergence bound}
Since our algorithm exactly simulates the reversed process, from Proposition \ref{pathKL}, the KL divergence between $p(\delta)$ and $\hat{p}^\leftarrow(T-\delta)$ is bounded by the KL divergence between the two path measures:
\begin{equation}\label{TV-bound} \mathrm{KL}(p(\delta) \| \hat{p}^\leftarrow(T-\delta))   \le \mathrm{KL}(p(T) \| \mathrm{Unif}(\{0,1\}^d)) + \int_\delta^T \EE_{X_t \sim p(t)} \ell (c_{X_t}(t), s_{X_t}(t)),
\end{equation}
where $\ell: \RR^d \times \RR^d \to \RR$ is given by
$ \ell(c,s) =  \sum_{i=1}^d \left(-c_i+s_i + c_i\log \frac{c_i}{s_i}\right)$. \\
In the RHS of \eqref{TV-bound}, the first term is bounded by Proposition \ref{converge-forward} that
$$  \mathrm{KL}(p(T) \| \mathrm{Unif}(\{0,1\}^d)) \lesssim de^{-T};      $$
the second term is bounded by $T\epsilon$ under Assumption \ref{as1}. Thus, we obtain the desired bound. 
\paragraph*{Proof of the TV distance bound}
We bound the TV distance between the data distribution $p_0$ and the perturbed distribution $p_\delta$. Consider the forward process $(X_t)_{t\ge 0}$. Since $(X_0, X_\delta)$ gives a coupling of $p_0$ and $p_\delta$, we have
$$ \mathrm{TV}(p_0,p_\delta) \le \PP(X_0 \ne X_\delta). $$
$\PP(X_0 \ne X_\delta)$ equals to the probability that a $\mathrm{Poisson}(d\delta)$ random variable is nonzero, which is $1-e^{-d\delta}$. Thus by triangle inequality and Pinsker's inequality $\mathrm{TV} \lesssim  \sqrt{\mathrm{KL}}$, we have
$$\mathrm{TV}(p_0,\hat{p}^\leftarrow(T-\delta))  \le  d\delta +  \mathrm{TV}(p_\delta,\hat{p}^\leftarrow(T-\delta)) \lesssim (1-e^{-d\delta}) + \sqrt{T\epsilon + de^{-T}}. $$
We complete the proof.
\qed

\subsection{Proof of Theorem \ref{bounded-ratio}} \label{proof-boundedratio}
\begin{lemma} \label{bound-score-smooth}
Let $p(t)$ be the distribution of the forward CTMC with generator $Q$ given in \eqref{OU-generator}. Suppose Assumption \ref{as3} holds. For any $t>0$ we have
$$\frac{p_{x+e_i}(t)}{p_x(t)} \le L.$$
\end{lemma}
\begin{proof}
    Let $g$ be the heat kernel defined in Proposition \ref{explicit-transition}. Consider the conditional distribution $\tilde{p}(a|x) \propto p_a(0)g_{x-a}(t)$ that is the distribution of $X_t$ conditional on $X_0$. We write the probability ratio as 
    \begin{align*}
 \frac{p_{x+e_i}(t)}{p_x(t)} & = \frac{\sum_{a+w=x+e_i}p_a(0)g_w(t) }{\sum_{a+w=x}p_a(0)g_w(t)} \\
 &  = \frac{\sum_{a+w=x}p_{a+e_i}(0)g_{w}(t) }{\sum_{a+w=x}p_a(0)g_w(t)} \\
 &  =  \frac{\sum_{a+w=x}p_a(0)\frac{p_{a+e_i}(t)}{p_a(t)}g_w(t) }{\sum_{a+w=x}p_a(0)g_w(t)} \\
 &  = \EE_{\tilde{p}(a|x)} \frac{p_{a+e_i}(0)}{p_a(0)} \\
 &  \le L,
\end{align*}
where the last inequality comes from Assumption \ref{as3}. We complete the proof.
\end{proof}
\paragraph*{Proof of Theorem 4.7} 
The KL divergence bound is similar to the proof of Theorem 4.4. We only need to consider the total number of transitions.
We choose the time partition such that
$$ t_N = T,\,t_{N-1} = T- 1/L,\,t_{k+1} - t_k \le c(T-t_{k+1}),\,\forall t_{k+1}- t_k \le c(T-t_{k+1}),\, 1\le k \le N-2      $$
for some absolute constant $c$ and choose $\lambda_k$'s by
$$ \lambda_N =  \frac{d}{L},\, \lambda_k =\frac{Cd}{\min(1,T-t_k)},\, 1\le k \le N-1.         $$
Let $s_k = T- t_k$. The total number of steps follows a Poisson distribution with parameter $\lambda =\sum_{k=1}^N \lambda_k(t_k - t_{k-1}) $. We have
\begin{align*}
\lambda & = \lambda_N(t_{N}-t_{N-1}) +  \sum_{k=1}^{N-1} \lambda_k(t_k - t_{k-1}) \le 1 + \sum_{k: 1/L \le s_k < 1 }\frac{Cd}{s_k}(s_{k-1}-s_k) +  \sum_{k: 1\le s_k \le T} Cd(s_k - s_{k-1}) \\
& \lesssim d + \int_{1/L}^1 \frac{d}{s} \mathrm{d} s + \sum_{k: 1\le s_k \le T} d(s_k - s_{k-1}) \lesssim d\left(T+ \log L \right).
\end{align*}
We complete the proof.
\qed

\end{document}